\documentclass{article}

\pdfoutput=1

\usepackage{amsmath,amsthm,amssymb,amsfonts,amsthm}

\newtheorem{mydefinition}{Definition}
\newtheorem{proposition}[mydefinition]{Proposition}
\newtheorem{theorem}[mydefinition]{Theorem}

\newtheorem{lemma}[mydefinition]{Lemma}
\newtheorem{definition}[mydefinition]{Definition}

\newcommand{\M}{\mathcal{M}}

\usepackage{graphicx} 
\usepackage{subfigure}
\usepackage{natbib}
\usepackage{algorithm}
\usepackage{algorithmic}
\usepackage{hyperref}

\newcommand{\real}{\mathbb{R}}
\newcommand{\myvec}[1]{\mathbf{#1}}
\newcommand{\myvecsym}[1]{\boldsymbol{#1}}

\newcommand{\vphi}{\myvecsym{\phi}}
\newcommand{\vPhi}{\myvecsym{\Phi}}

\newcommand{\vtheta}{\myvecsym{\theta}}

\newcommand{\vTheta}{\myvecsym{\Theta}}

\newcommand{\vx}{\myvec{x}}

\newcommand{\expect}[1]{\mathbb{E}\left[ {#1} \right]}

\newcommand{\calC}{{\cal C}}

\newcommand{\be}{\begin{equation}}
\newcommand{\ee}{\end{equation}}
\newcommand{\bea}{\begin{eqnarray}}
\newcommand{\eea}{\end{eqnarray}}
\newcommand{\beaa}{\begin{eqnarray*}}
\newcommand{\eeaa}{\end{eqnarray*}}

\DeclareMathAlphabet{\mathpzc}{OT1}{pzc}{m}{n}

\newcommand{\keywordDef}[1]{{\emph{#1}}}

\usepackage[accepted]{icml2014}

\icmltitlerunning{Linear and Parallel Learning of Markov Random Fields}

\allowdisplaybreaks

\begin{document}

\twocolumn[
\icmltitle{Linear and Parallel Learning of Markov Random Fields}

\icmlauthor{Yariv Dror Mizrahi$^1$}{yariv@math.ubc.ca}
\icmlauthor{Misha Denil$^2$}{misha.denil@cs.ox.ac.uk}
\icmlauthor{Nando de Freitas$^{1,2}$}{nando@cs.ox.ac.uk}
\icmladdress{$^1$University of British Columbia, Canada\\
$^2$University of Oxford, United Kingdom}

\icmlkeywords{scaling, distributed learning, MRFs, exact inference, pseudo-likelihood, maximum likelihood, probabilistic graphical models, random fields}

\vskip 0.3in
]

\begin{abstract}
We introduce a new embarrassingly parallel parameter learning algorithm for Markov random fields with untied parameters which is efficient for a large class of practical models.  Our algorithm parallelizes naturally over cliques and, for graphs of bounded degree, its complexity is linear in the number of cliques. Unlike its competitors, our algorithm is fully parallel and for log-linear models it is also data efficient, requiring only the local sufficient statistics of the data to estimate parameters.
\end{abstract}



\section{Introduction}

Markov Random Fields (MRFs), also known as undirected probabilistic graphical models, are ubiquitous structured probability models that have  significantly impacted a large number of fields, including computer vision \cite{Li:2001,Szeliski:2008}, computational photography and graphics \cite{Agarwala:2004}, computational neuroscience \cite{Ackley:1985},  bio-informatics \cite{Yanover:2007}, sensor networks \cite{Liu:2012}, social networks \cite{Strauss:1990}, Markov logic \cite{Richardson:2006}, natural language processing \cite{Lafferty:2001,Sutton:2012} and statistical physics \cite{Kindermann:1980}. As pointed out in \citet{Wainwright:2008} there are also many applications in statistics, constraint satisfaction and combinatorial optimization, error-correcting codes and epidemiology. Not surprisingly, many comprehensive treatments of this important topic have appeared in the last four decades 
\cite{Kindermann:1980,Lauritzen:1996,Bremaud:2001,Koller:2009,Murphy:2012}.

Despite the great success and impact of these models, fitting them to data remains a formidable challenge. Although the log-likelihood  is typically convex in the parameters, the gradient of these models is intractable.

In many cases, maximum likelihood in these models is \emph{data efficient} in the sense that the data term in the gradient can be easily precomputed, making its evaluation trivial during optimization. The main difficulty with maximum likelihood is that it is not \emph{model efficient} since evaluating the gradient involves computing expectations over the model distribution.  This requires evaluating a sum with exponentially many terms, which is intractable for even moderately sized models. The intractability of exact maximum likelihood has prompted the introduction of many approximate methods of parameter estimation \cite{Besag:1975,Hinton:2000,Hyvarinen:2005,Marlin:2010,Varin:2011,Marlin:2011,Swersky:2011}.

An important class of approximate method for this problem are stochastic approximation methods, which approximate the model term by drawing samples from the model distribution, typically via MCMC. This simulation is costly and often many samples are required for accurate estimation. Moreover, in settings where the parameters or data must be distributed across many machines such simulation poses additional difficulties.

Another approach is to approximate the maximum likelihood objective with a factored alternative.  The leading method in this area is pseudo-likelihood.  In this approach the joint distribution over all variables in the MRF is replaced by a product of conditional distributions for each variable.  Replacing the joint distribution with a product of conditionals eliminates the model term from the gradient of the pseudo-likelihood objective, which circumvents the model inefficiency of maximum likelihood estimation.  However, pseudo-likelihood is not data efficient, since the conditional distributions often depend on the actual data and the current value of the parameters.  We return to this issue in more detail in Section~\ref{sec:model-data-efficiency}.

Applying pseudo likelihood in a distributed setting is also difficult, because the conditional distributions share parameters.  Several researchers have addressed this issue by proposing to approximate pseudo-likelihood by disjointly optimizing each conditional and combining the parameters using some form of averaging \cite{Ravikumar:2010,Wiesel:2012,Liu:2012}.

In this paper we introduce a new approach to parameter estimation in MRFs with untied parameters, which avoids the model inefficiency of maximum likelihood for an important class of models while preserving its data efficiency.  Moreover, our algorithm is embarrassingly parallel and can be implemented in a distributed setting without modification. Our algorithm replaces the joint maximum likelihood problem with a collection of much smaller auxiliary maximum likelihood problems which can be solved independently.

We prove that if the auxiliary problems satisfy certain conditions, the relevant parameters in the auxiliary problems converge to the values of the true parameters in the joint model.  Our experiments show that good performance is achieved in this case and that good performance is still achieved when these conditions are not satisfied.  Violating the conditions for convergence sacrifices theoretical guarantees in exchange for even further computational savings while maintaining good empirical performance.

Under a strong assumption (which is generally not satisfied in practice) we prove that our algorithm is exactly equal to maximum likelihood on the full joint distribution.  While not directly applicable, this result provides additional insight into why our approach is effective.

A similar method was recently, and independently, introduced in the context of \emph{Gaussian graphical models} by \citet{Meng:2013}. In that paper, the authors consider local neighborhoods of nodes, whereas we consider neighborhoods of cliques, and they rely on a convex relaxation via the \emph{Schur complement} to derive their algorithm for inverse covariance estimation. At the time of revising this paper, the same authors have shown that the convergence rate to the true parameters with their method is comparable to centralized maximum likelihood estimation \cite{Meng:2014}. 

Although our work and that of Meng arrive at distributed learning algorithms via different paths, and while Meng et al.\ consider only Gaussian graphical models, it is clear that both works show that it is possible to capitalize on graph structures beyond low tree width to design algorithms that are both data and model efficient and exhibit good empirical performance.


\section{Model Specification and Learning Objectives}

We are interested in estimating the parameter vector $\vtheta$ of a positive distribution $p(\vx\,|\,\vtheta)>0$ that satisfies the Markov properties of an  undirected graph $G$. That is, a distribution that
can be represented as a product of factors, one
per maximal clique, 
\begin{align}
p(\vx\,|\,\vtheta) = \frac{1}{Z(\vtheta)} \prod_{c \in \calC} \psi_c(\vx_c\,|\,\vtheta_c),
\end{align}
where
 $\calC$ is the set of maximal cliques of $G$,
$\psi_c(\vx_c\,|\,\vtheta_c)\geq 0$ is the \keywordDef{potential function} or \keywordDef{factor}
associated with the variables in clique $c$,
and
$Z(\vtheta)$ is
the \keywordDef{partition function}
given by
\begin{align}
Z(\vtheta) = \sum_{\vx} \prod_{c \in \calC} \psi_c(\vx_c\,|\,\vtheta_c).
\end{align}
In such models we often use exponential functions to represent the potentials,
\begin{align}
\psi_c(\vx_c\,|\,\vtheta_c)  = \exp(-E(\vx_c\,|\,\vtheta_c)),
\end{align}
where $E(\vx_c\,|\,\vtheta_c) \in \real$ is called the \keywordDef{energy}, which we will assume is chosen so that the parameters are identifiable.
The resulting joint distribution can then be written as
a \keywordDef{Gibbs distribution}
\begin{align*}
p(\vx\,|\,\vtheta) = \frac{1}{Z(\vtheta)} \exp(-\sum_c E(\vx_c\,|\,\vtheta_c)).
\end{align*}
When the energy is a linear function of the parameters, i.e.\ 
$E(\vx_c\,|\,\vtheta_c) = -\vtheta_c^T \vphi_c(\vx_c)$
where $\vphi_c(\vx_c)$  is a feature vector derived from the values of the
variables $\vx_c$,
we have a \keywordDef{maximum entropy} or \keywordDef{log-linear} model \cite{Wasserman:2004,Buchman:2012,Murphy:2012}. The features in these models are also referred to as local sufficient statistics.

\subsection{Maximum Likelihood}

There is (in general) no closed form solution
for the ML estimate of the parameters of an MRF, so gradient-based optimizers are needed.

Consider the fully-observed maximum entropy model
\begin{align}
p(\vx\,|\,\vtheta) = 
\frac{1}{Z(\vtheta)} \exp(\sum_c \vtheta_c^T \vphi_c(\vx))
\end{align}
where $c$ indexes the maximal cliques.
The scaled log-likelihood is given by
\begin{align*}
\ell(\vtheta)  &=
\frac{1}{N} \sum_{n=1}^{N} \log p(\vx_n\,|\,\vtheta)
 \\
  &= \frac{1}{N} \sum_{n=1}^{N} \left[ \sum_c \vtheta_c^T \vphi_c(\vx_n) -
 \log Z(\vtheta) \right]
\end{align*}
which is a convex function of $\vtheta$.

The derivative for the parameters of a particular clique, $q$,
is given by
\begin{align}
\frac{\partial \ell}{\partial \vtheta_q}
 &= \frac{1}{N} \sum_{n=1}^{N} \left[ \vphi_q(\vx_{n}) -
 \frac{\partial}{\partial \vtheta_q} \log Z(\vtheta)
\right] \enspace,
\end{align}
where
\begin{align}
\frac{\partial}{\partial \vtheta_q} \log Z(\vtheta)
 =
\expect{\vphi_q(\vx)\,|\,\vtheta}
 =   \sum_{\vx} \vphi_q(\vx) p(\vx\,|\,\vtheta) \enspace.
\label{eqn:derivLogZ}
\end{align}
Equation~\ref{eqn:derivLogZ} is the expectation of the feature $\vphi_q(\vx)$ over the model distribution.  For many models of interest this quantity is intractable.

The full derivative of the log-likelihood contrasts the model expectation against the expected value of the feature over the data,
\begin{align}
\frac{\partial \ell}{\partial \vtheta_q}
 = \frac{1}{N} \sum_{n=1}^{N} 
\vphi_q(\vx_n) - \expect{\vphi_q(\vx)\,|\,\vtheta} \enspace.
\label{eq:gradml}
\end{align}
At the optimum these two terms will be equal and the empirical distribution of the features will match
the model predictions.

\subsection{Maximum Pseudo-Likelihood}

To surmount the intractable problem of computing expectations over the model distribution, pseudo-likelihood considers a simpler factorised objective function,
\begin{align}
\ell^{PL}(\vtheta)  =
\frac{1}{N} \sum_{n=1}^{N} \sum_{m=1}^{M} \log p(x_{mn}\,|\,\vx_{-mn},\vtheta)
\end{align}
where $\vx_{-mn}$ denotes all the components of the $n$-th data vector, except for component $m$. (For models with sparse connectivity, we only need to condition on the neighbors of node $m$.)
In the binary, log-linear case, the gradient of this objective can be expressed in contrastive form,
\begin{align*}
\frac{\partial \ell^{PL}}{\partial \vtheta_q}
 =  \frac{1}{N}  \sum_{n,m}  p(\bar{x}_{mn}^m\,|\,\vx_{-mn},\vtheta) 
 \left[
  \vphi_q(\vx_n) - 
  \vphi_q(\bar{\vx}_n^m)
\right]\enspace,
\label{eqn:gradpl}
\end{align*}
where $\bar{\vx}_n^m$ is the data vector $\bar{\vx}_n$ with the $m$-th bit flipped. That is,
$\bar{x}_{mn}^i = 1 - x_{mn}$ if $i=m$ and $x_{mn}$ otherwise \cite{Marlin:2010}. 

\subsection{Model and Data Efficiency}
\label{sec:model-data-efficiency}

There are two terms in the gradient of Equation~\ref{eq:gradml}.  The first term is an empirical expectation, $\frac{1}{N} \sum_{n=1}^{N}\vphi_q(\vx_n)$, and depends only on the data.  The value of this term for each clique can be pre-computed before parameter optimization begins, making this term of the gradient extremely cheap to evaluate during optimization.

The data term in the maximum likelihood gradient is contrasted with an expectation over the model distribution, $\expect{\vphi_q(\vx)\,|\,\vtheta}$, which is a sum over exponentially many configurations.  For large models this term is intractable.

We describe this situation by saying that maximum likelihood estimation is data efficient, since the terms involving only the data can be computed efficiently.  However, maximum likelihood is not model efficient, since the model term in the gradient is intractable, and the difficulty in evaluating it is the primary motivation for the development of alternative objectives like pseudo-likelihood.

Pseudo-likelihood addresses the model inefficiency of maximum likelihood by eliminating the model term from the gradient, which makes pseudo-likelihood model efficient.  However, pseudo-likelihood is not data efficient, since computing the gradient requires access to the full conditional distributions $p(\bar{x}_{mn}^m\,|\,\vx_{-mn},\vtheta)$.  Because of this the outer sum over data examples must be computed for each gradient evaluation. (Note that for binary models the full conditionals correspond to logistic regressions, so any advances in scaling logistic regression to massive models and datasets would be of use here.)

In the following section we introduce the LAP algorithm (which stands for Linear and Parallel), which uses a particular decomposition of the graph to avoid the exponential cost in maximum likelihood, but unlike pseudo-likelihood our algorithm is fully parallel and maintains the data efficiency of maximum likelihood estimation.  Our algorithm divides the full parameter estimation process into several fully independent sub-problems which can be solved in parallel.  Once each sub-problem has been solved the solutions of the sub-problems are combined to give a solution to the full problem.


\section{Algorithm Description}

The LAP algorithm operates by splitting the joint parameter estimation problem into several independent sub-problems which can be solved in parallel.  Once the sub-problems have been solved, it combines the solutions to each sub-problem together into a solution to the full problem.

For a fixed clique $q$ we define its 1-neighborhood
\begin{align*}
A_q = \bigcup_{c \cap q \neq \emptyset} c
\end{align*}
to contain all of the variables of $q$ itself as well as the variables with at least one neighbor in $q$.

LAP creates one sub-problem for each maximal clique in the original problem by defining an \emph{auxiliary MRF}, $\M_q$, over the variables in $A_q$.  Details on how to construct the auxiliary MRF will be discussed later, for now we assume we have an auxiliary MRF on $A_q$ and that it contains a clique over the variables in $q$ that is parametrized the same way as $q$ in the original problem.

LAP derives the parameter vector $\vtheta_q$ for the full problem by estimating parameters in the auxiliary MRF on $A_q$ using maximum likelihood and reading off the parameters for the clique $q$ directly.  The steps of the algorithm are summarized in Algorithm~\ref{alg:lap}.

\begin{algorithm}[tb]
   \caption{LAP}
   \label{alg:lap}
\begin{algorithmic}
   \STATE {\bfseries Input:} MRF with maximum cliques $\mathcal{C}$
   \FOR{$q \in \mathcal{C}$}
   \STATE Construct auxiliary MRF $\M_q$ on variables in $A_q$
   \STATE Estimate parameters $\vtheta^{\M_q}$ in $\M_q$ using ML
   \STATE Set $\vtheta_q \leftarrow \vtheta_q^{\M_q}$
   \ENDFOR
\end{algorithmic}
\end{algorithm}

\begin{figure}[t!]
\centering
\includegraphics{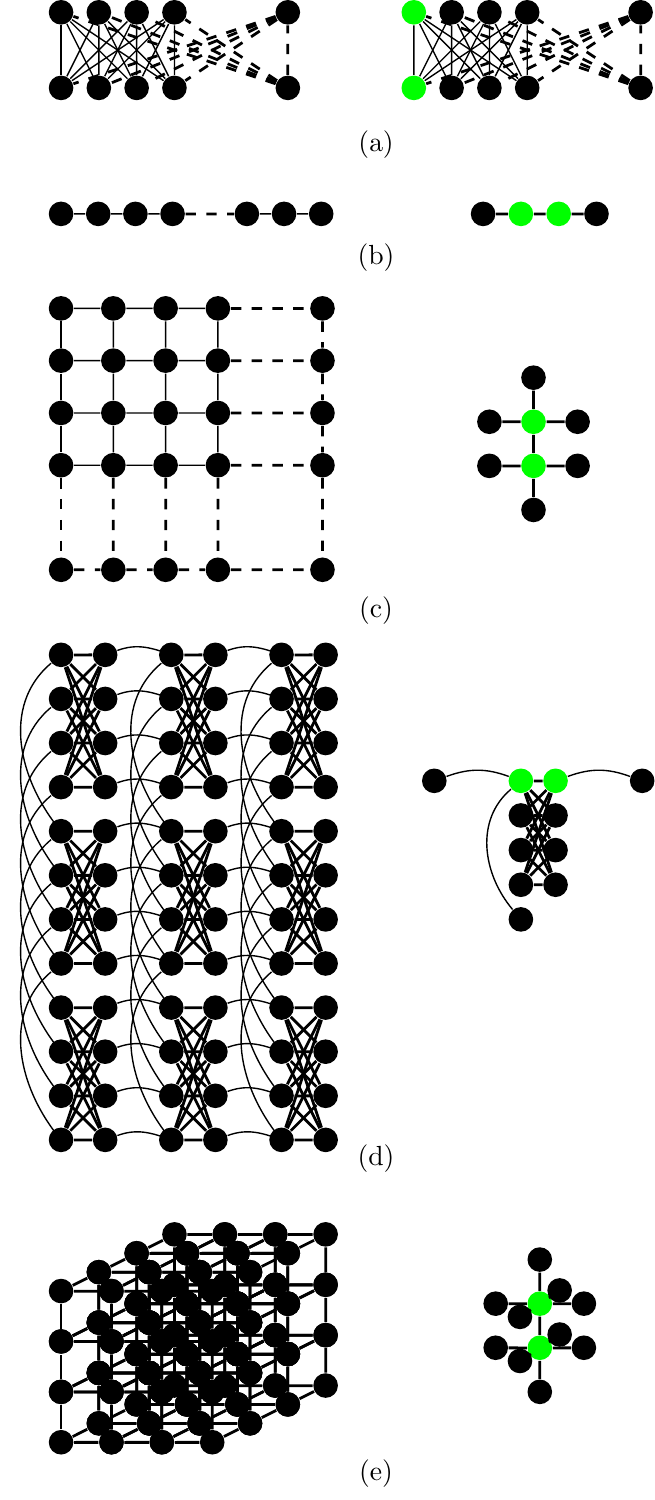} 
\caption{The left column shows several popular MRFs:
(a) a restricted Boltzmann machine (RBM),   (b) a chain graph, (c) a  2-D Ising  grid, (d) a Chimera $3\times 3\times 4$ lattice, and  (e) a 3-D Ising lattice. The right hand side shows the corresponding 1-neighborhoods for cliques of interest (in green). Models (b) to (e) have small 1-neighborhoods and can learned efficiently with the LAP algorithm.}
\label{fig:mrfs}
\end{figure}

In a log-linear model, when estimating the vector of parameters $\vtheta^{\M_q}$ of the auxiliary MRF by maximum likelihood, the relevant derivative is
\begin{align*}
\frac{\partial \ell^{\M_q}}{\partial \vtheta^{\M_q}_q}
 = \frac{1}{N} \sum_{n=1}^{N}  
\vphi_q(\vx_{A_qn}) - \expect{\vphi_q(\vx_{A_q}) |\vtheta^{\M_q}} \enspace.
\end{align*}
This approach is data efficient, since 
the sufficient statistics $\frac{1}{N} \sum_{n=1}^{N}\vphi_q(\vx_{A_qn}) $ can be easily pre-computed. Moreover, the data vector $\vx_n$ can be stored in a distributed fashion, with the node estimating the MRF $\M_q$ only needing access to the sub-vector $\vx_{A_qn}$. In addition, 
LAP is model efficient since 
the expectation $\expect{\vphi_q(\vx_{A_q}) |\vtheta^{\M_q}}$ can be easily computed when the number of variables in $A_q$ is small. To illustrate this point, consider the models shown in Figure~\ref{fig:mrfs}. 
For dense graphs, such as the restricted Boltzmann machine, the exponential cost of enumerating over all the variables in $A_q$ is prohibitive. However, for other practical MRFs of interest, including lattices and Chimeras \cite{Denil:2011}, this cost is perfectly acceptable.

\subsection{Construction of the Auxiliary MRF}

The effectiveness of LAP comes from proper construction of the auxiliary MRF, $\M_q$.  As already mentioned, $\M_q$ must contain the clique $q$, which must be parametrized in the same way as in the joint model.  This requirement is clear from the previous section, otherwise the final step in Algorithm~\ref{alg:lap} would be invalid.

We will see in the analysis section that it is desirable for $\M_q$ to be as close to the marginal distribution on $\vx_{A_q}$ as possible.  This means we must include all cliques from the original MRF which are subsets of $A_q$.  Additionally, marginalization may introduce additional cliques not present in the original joint distribution.  It is clear that these cliques can only involve variables in $A_q\setminus q$, but determining their exact structure in general can be difficult.

We consider three strategies for constructing auxiliary MRFs, which are distinguished by how they induce clique structures on $A_q\setminus q$.  The three strategies are as follows.

\paragraph{Exact} Here we compute the exact structure of the marginal distribution over $A_q$ from the original problem.  We have chosen our test models to be ones where the marginal structure is readily computed.


\paragraph{Dense} For many classes of model the marginal over $A_q$ involves a fully parametrized clique over $A_q\setminus q$ for nearly every choice of $q$ (for example, this is the case in lattice models).  The dense variant assumes that the marginal always has this structure.  Making this choice will sometimes over-parametrize the marginal, but avoids the requirement of explicitly computing its structure. 

\paragraph{Pairwise}  Both the exact and dense strategies create high order terms in the auxiliary MRF.  While high order terms do exist in the marginals of discrete MRFs, it is computationally inconvenient to include them, since the add many parameters to each sub-problem.  In the pairwise variant we use the same graph structure as in dense, but here we introduce only unary and binary potentials over $A_q\setminus q$.  This results in a significant computational savings for each sub-problem in LAP, but fails to capture the true marginal distribution in many cases (including all of the example problems we consider).


\section{Experiments}

\begin{figure}[t]
\centering
\includegraphics[width=0.49\linewidth]{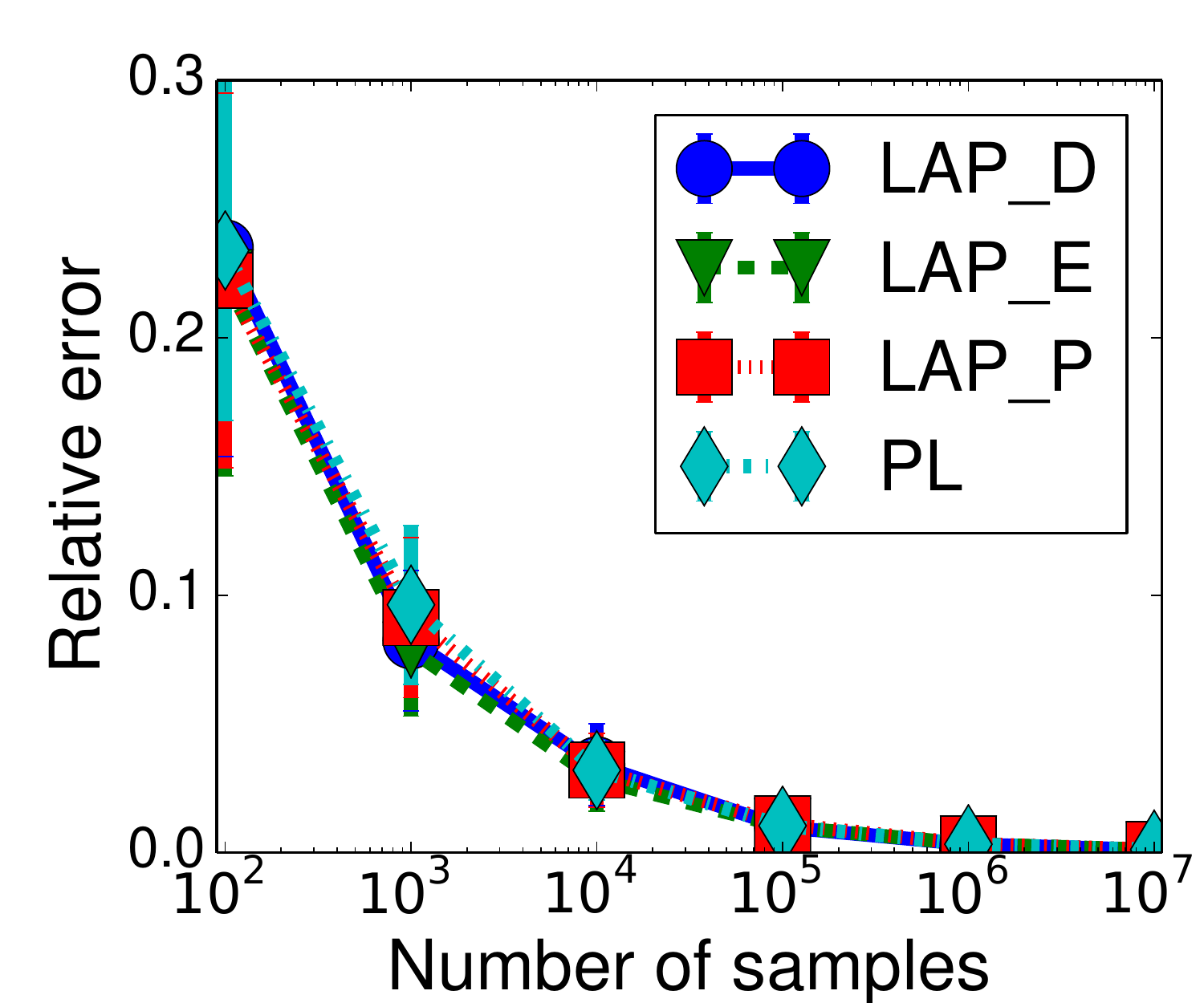}
\includegraphics[width=0.49\linewidth]{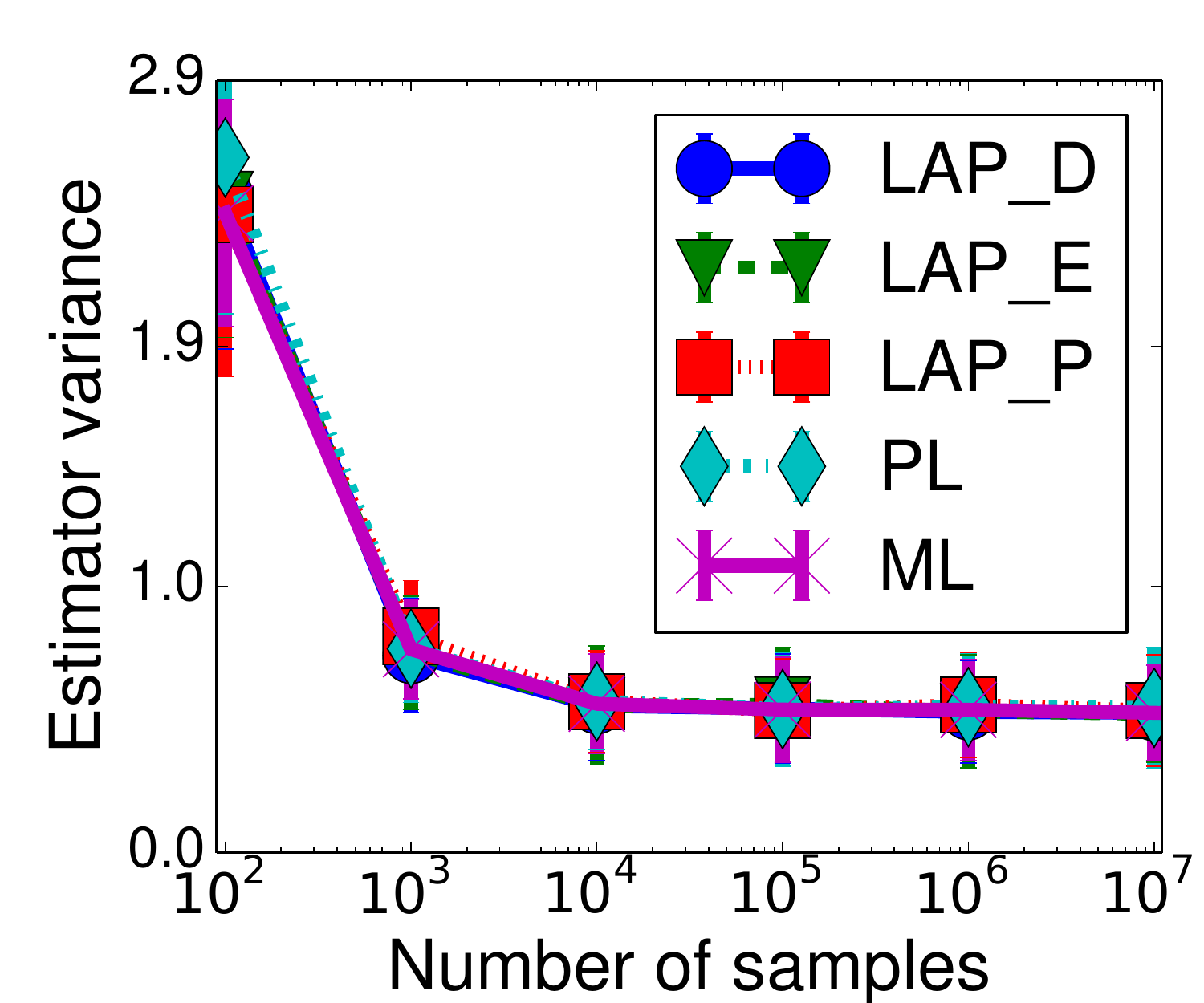}
\caption{\textbf{Left:} Relative error of parameter estimates compared to maximum likelihood for LAP and pseudo-likelihood on a $4\times 4$ Ising grid.  Error bars show the standard deviation over several runs.  \textbf{Right:} Variance of the parameter estimates for each algorithm.}
\label{fig:4x4}
\end{figure}

In this section we describe some experiments designed to show that the LAP estimator has good empirical performance.  We focus on small models where exact maximum likelihood is tractable in order to allow performance to be measured.  We chose to focus our experiments on demonstrating accuracy rather than scalability since the scaling and data efficiency properties of LAP are obvious.

The purpose of the experiments in this section is to show two things:
\begin{enumerate}
\item The accuracy of LAP estimates is not worse than its main competitor, pseudo-likelihood; and
\item LAP achieves good performance even when the exact marginal structure is not used.
\end{enumerate}

In all of our experiments we compare pseudo-likelihood estimation against LAP using the three different strategies for constructing the auxiliary MRF discussed in the previous section.  In each plot, lines labeled \texttt{PL} correspond to pseudo-likelihood and \texttt{ML} corresponds to maximum likelihood.  \texttt{LAP\_E}, \texttt{LAP\_D} and \texttt{LAP\_P} refer respectively to LAP with the exact, dense and pairwise strategies for constructing the auxiliary MRF.

We compare LAP and pseudo-likelihood to maximum likelihood estimation on three different model classes. The first is a $4\times 4$ Ising grids with 4-neighborhoods, and the results are shown in Figure~\ref{fig:4x4}.  The second is a $4\times 4\times 4$ Ising lattice with 6-neighborhoods, which is shown in Figure~\ref{fig:4x4x4}.  Finally, we also consider a Chimera $3\times 3\times 3$ model, with results shown in Figure~\ref{fig:chimera-3x3x3}.

The procedure for all models is the same: we choose the generating parameters uniformly at random from the interval $[-1, 1]$ and draw samples approximately from the model.  We then fit exact maximum likelihood parameters based on these samples, and compare the parameters obtained by pseudo-likelihood and LAP to the maximum likelihood estimates.  The left plot in each figure shows the mean relative error of the parameter estimates using the maximum likelihood estimates as ground truth.  Specifically, we measure
\begin{align*}
\operatorname{err}(\theta) = \|\theta^{ML}\|^{-1}\cdot\|\theta - \theta^{ML}\|
\end{align*}
for each estimate on each set of samples and average over several runs.

\begin{figure}[t]
\centering
\includegraphics[width=0.49\linewidth]{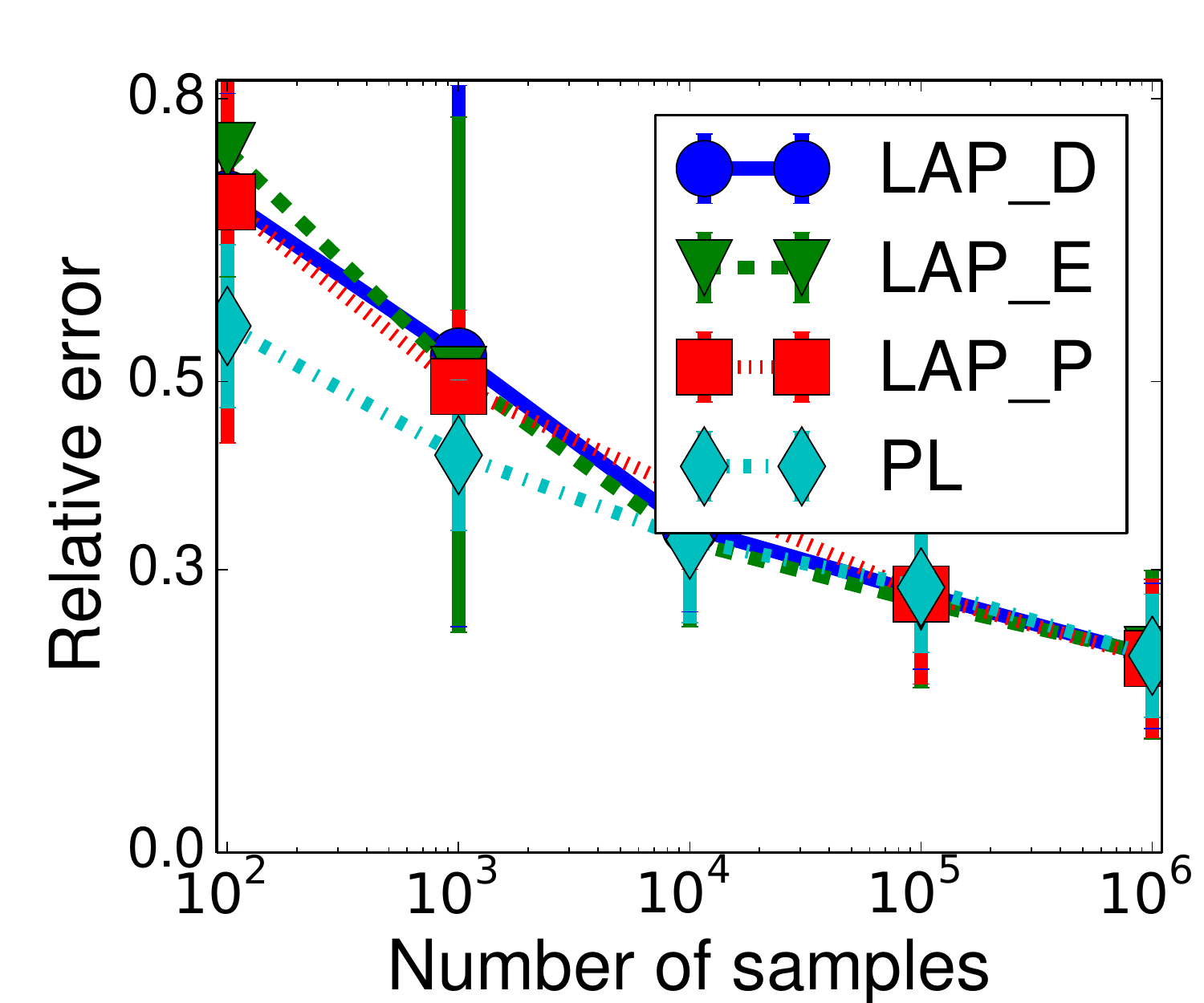}
\includegraphics[width=0.49\linewidth]{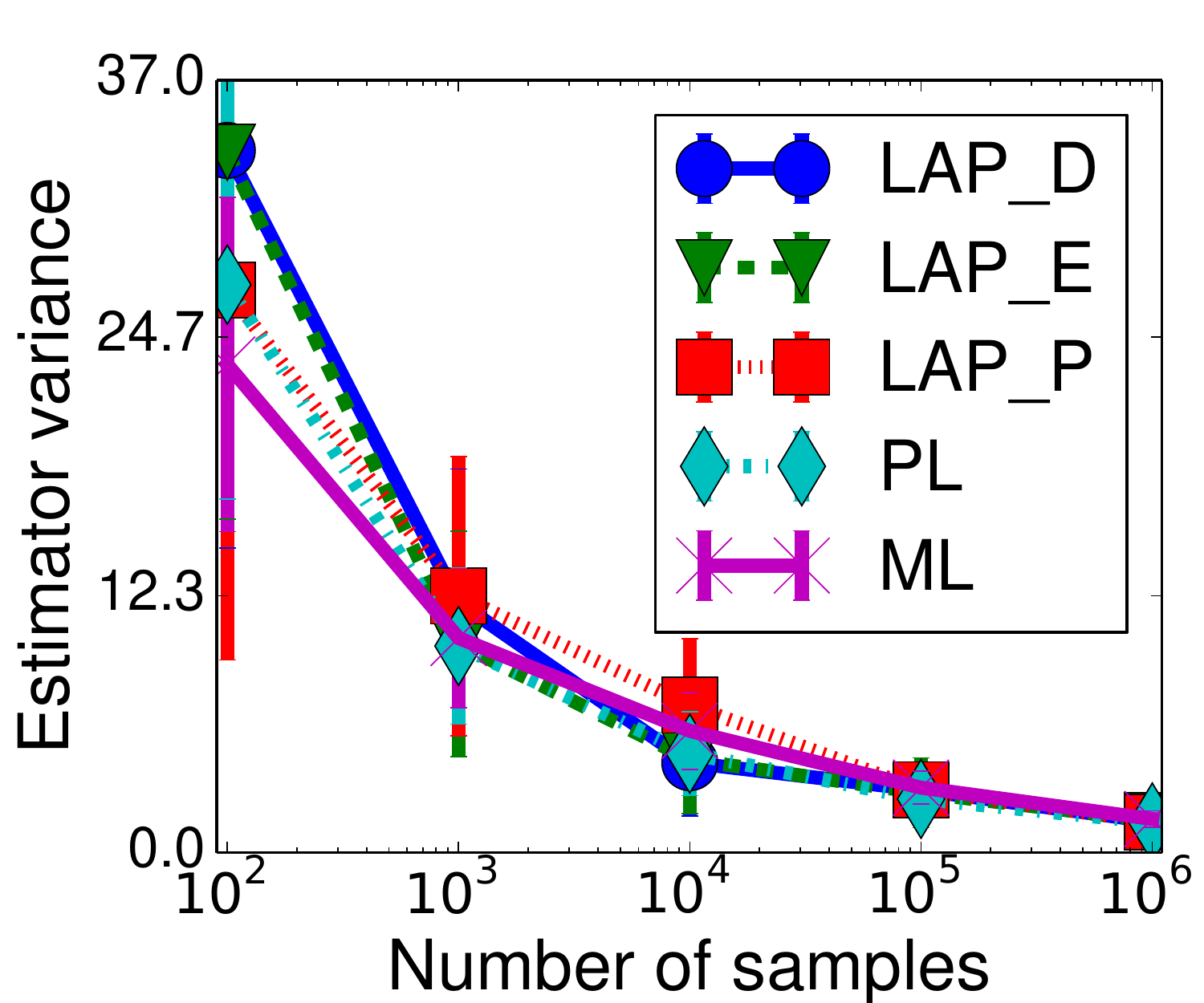}
\caption{\textbf{Left:} Relative error of parameter estimates compared to maximum likelihood for LAP and pseudo-likelihood on a $4\times 4\times 4$ Ising lattice.  Error bars show the standard deviation over several runs.  \textbf{Right:} Variance of the parameter estimates for each algorithm.}
\label{fig:4x4x4}
\end{figure}

We also measure the variance of the estimates produced by each algorithm over several runs. In this case we measure the variance of the estimates of each parameter separately and average these variances over all parameters in the model.  These measurements are shown in the right plot in each figure.  For reference we also show the variance of the maximum likelihood estimates in these plots.

In all of the experiments we see that the performance of all of the LAP variants is basically indistinguishable from pseudo-likelihood, except for small numbers of samples.  Interestingly, \texttt{LAP\_P} does not perform noticeably worse than the other LAP variants on any of the problems we considered here.  This is interesting because \texttt{LAP\_P} approximates the marginal with a pairwise MRF, which is not sufficient to capture the true marginal structure in any of our examples.  \texttt{LAP\_P} is also the most efficient LAP variant we tested, since the auxiliary MRFs it uses have the fewest number of parameters.


\section{Theory}

In this section show that matching parameters in the joint and the marginal distributions is valid, provided the parameterizations are chosen correctly. 
We then prove consistency of the LAP algorithm and illustrate its connection to maximum likelihood.

Undirected probabilistic graphical models can be specified, locally, in terms of Markov properties and conditional independence and, globally, in terms of an energy function $\sum_c E(\vx_c|\vtheta_c)$.  The Hammersley-Clifford theorem~\cite{hammersley1971markov} establishes the equivalence of these two representations.

One important fact that is often omitted is that the energy function and the partition function are not unique. 
It is however possible to obtain uniqueness, for both of these functions, by imposing normalization with respect to a setting of the random variables of the potential. This gives rise to the concept of \emph{normalized potential} \cite{Bremaud:2001}:
\begin{definition}
\label{norm_potential}
 A Gibbs potential $E(\vx_c|\vtheta_c)$ is said to be normalized with respect to zero if $E(\vx_c|\vtheta_c) = 0$ whenever there exists $t \in c$ such that $\vx_{t} = 0$.
 \end{definition}  
(In this section, we use the term \emph{Gibbs potential}, or simply \emph{potential}, to refer to the energy so as to match the nomenclature of \cite{Bremaud:2001}.)
The following theorem plays a central role in understanding the LAP algorithm. 
The proof can be found in \cite{Griffeath:1976,Bremaud:2001}: 
 \begin{theorem}
\emph{\textbf{[Existence and Uniqueness of the normalized potential]}}
\label{thm:np}
There exists one and only one (Gibbs) potential normalized with respect to zero corresponding to a Gibbs distribution.
\label{thm:uniqueness}
\end{theorem}

\begin{figure}[t!]
\centering
\includegraphics[width=0.49\linewidth]{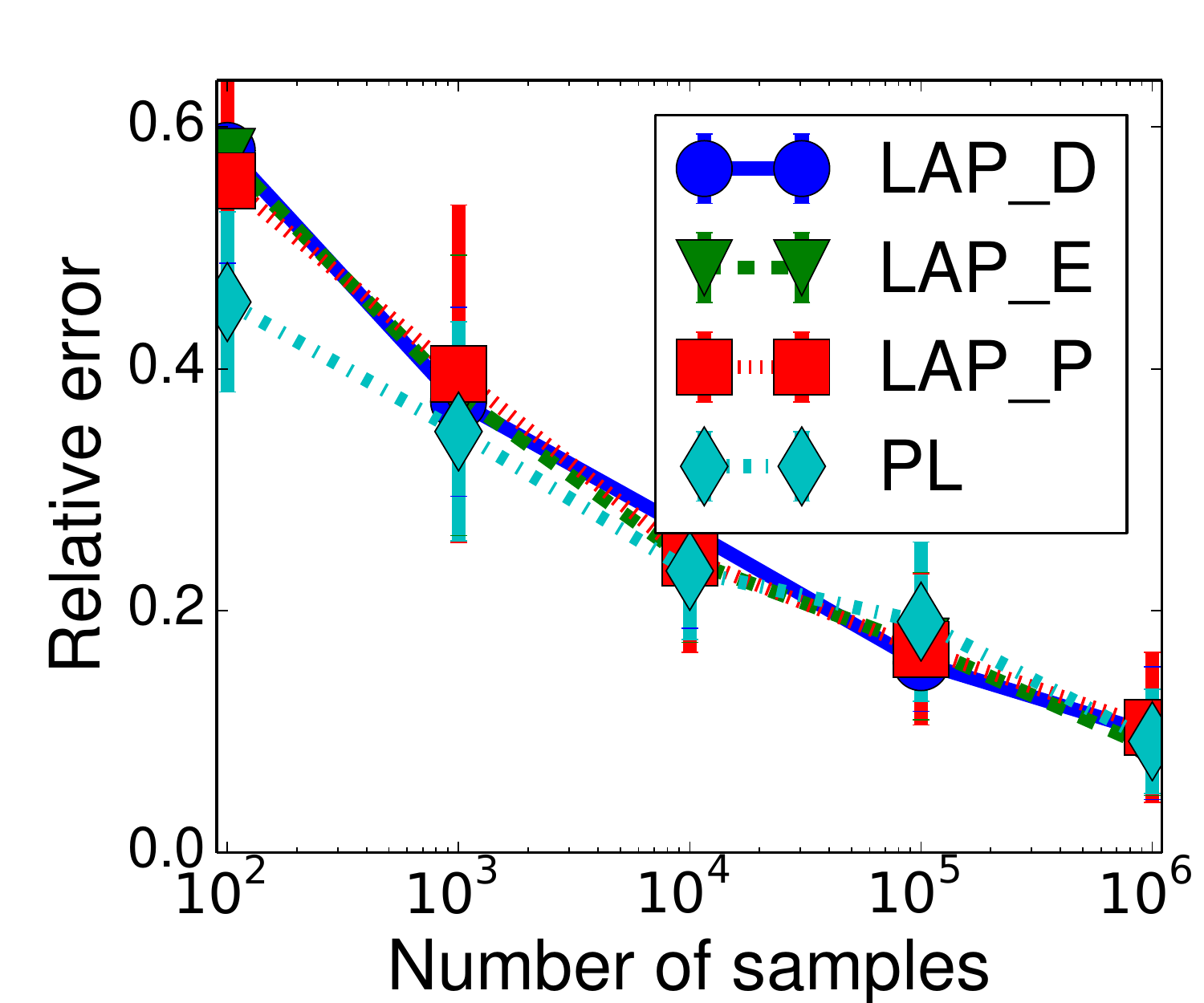}
\includegraphics[width=0.49\linewidth]{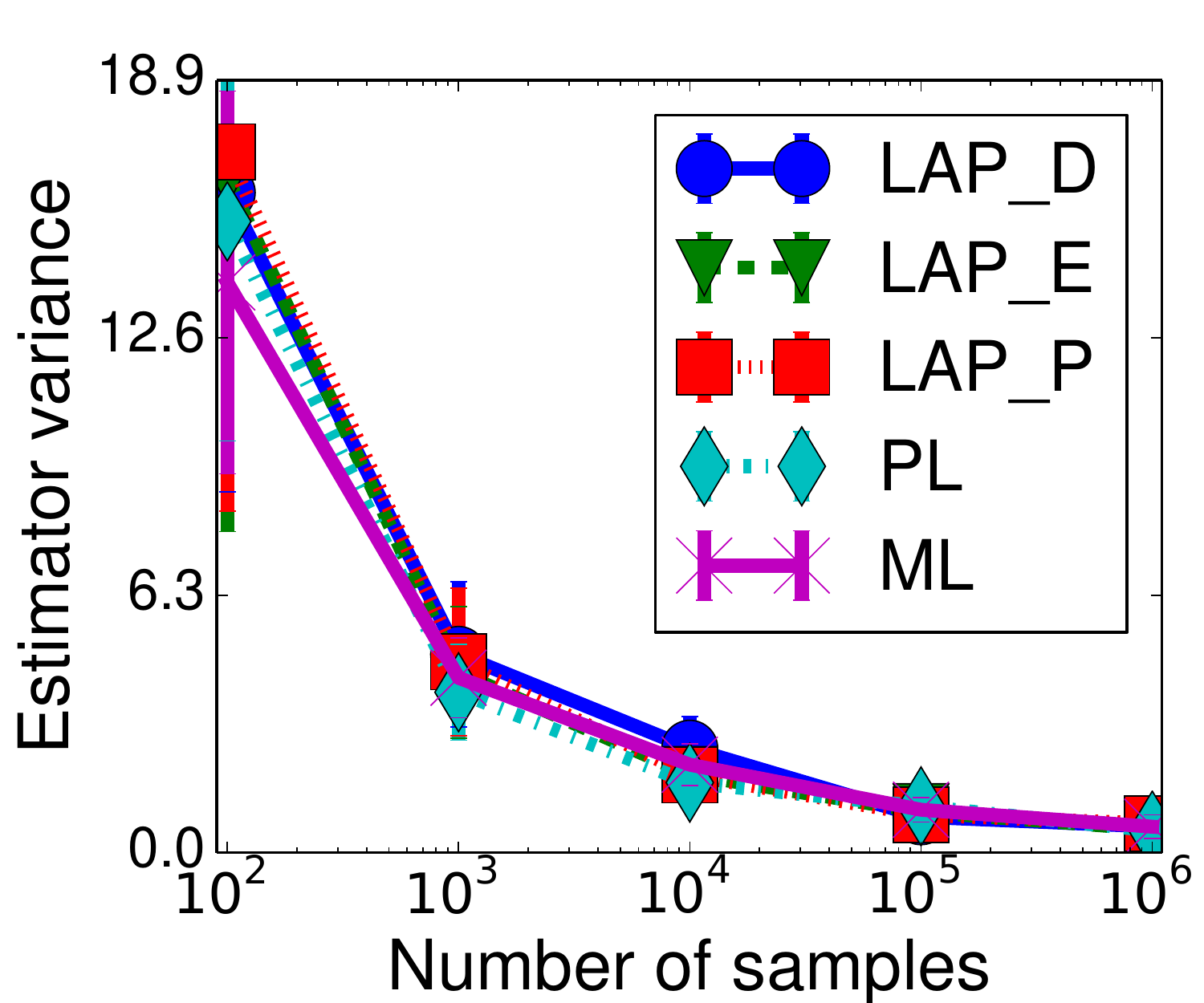}
\caption{\textbf{Left:} Relative error of parameter estimates compared to maximum likelihood for LAP and pseudo-likelihood on a Chimera $3\times 3\times 3$ model.  Error bars show the standard deviation over several runs.  \textbf{Right:} Variance of the parameter estimates for each algorithm.}
\label{fig:chimera-3x3x3}
\end{figure}

Earlier in this paper we used $\vx$ with no subscript to refer to the vector of all variables in the MRF since there was no risk of confusion.  In this section we increase the precision in our notation by using $S$ to denote the set of all variables and use $\vx_S$ instead of $\vx$ for the vector of all variables in the MRF.

\subsection{The LAP Argument}

Suppose we have a Gibbs distribution $p(\vx_S\,|\,\vtheta)$ which factors according to the clique system $\mathcal{C}$, and let $q\in\mathcal{C}$ be a particular clique of interest.  Let 
\begin{align*}
p(\vx_{A_q}\,|\, \vphi) &= \frac{1}{Z(\vphi)}\exp (- \sum_{c\in\mathcal{C}_q} E(\vx_c\,|\, \vphi_c) )
\end{align*}
be the marginal distribution on $A_q$ (with clique system $\mathcal{C}_q$) parametrized so that the potentials are normalized with respect to zero.

We can also write the marginal in the following way
\begin{align*}
p(\vx_{A_q}\,|\,\vtheta) &= \sum_{S\setminus A_q} p(\vx_S\,|\,\vtheta)
\\
&= \frac{1}{Z(\vtheta)} \sum_{S\setminus A_q}\exp ( -\sum_{c\in\mathcal{C}} E(\vx_c \,|\, \vtheta_c) )
\\
&= \frac{1}{Z(\vtheta)}\exp ( -E(\vx_q\,|\, \vtheta_q) - \hspace{-3mm}\sum_{c\in\mathcal{C}_q \setminus \{q\}}  \hspace{-3mm} E(\vx_c\,|\, \vtheta_{S\setminus q} ) )
\end{align*}
If the parametrization of $p(\vx_S\,|\,\vtheta)$ is also chosen to be normalized with respect to zero then the potentials of these two representations of the marginal must be equal. Theorem~\ref{thm:uniqueness} also tells us that the partition functions must be equal.  In particular we see that
\begin{align*}
 E(\vx_q\,|\, \vphi_q) = E(\vx_q\,|\, \vtheta_q)
\end{align*}
which implies that $\vtheta_q = \vphi_q$ if the parameters are identifiable.

\subsection{Consistency of LAP}

Let $\vtheta$ be the true vector of parameters taken from the unknown generating distribution $p(\vx_S \,|\, \vtheta)$ parametrized such that the potentials are normalized with respect to zero.  Suppose we have $N$ samples drawn \textit{iid} from this distribution.  Let $\hat{\vtheta}^{ML}$ be the ML parameters for $\vtheta$ given the data and let $\hat{\vtheta}^{LAP}$ be the corresponding LAP estimate.  We claim that $\hat{\vtheta}^{LAP} \to \vtheta$ as $N\to\infty$, provided the true marginal distributions are contained in the class of auxiliary MRFs.
\begin{proof}
Let $q \in \mathcal{C}$ be an arbitrary clique of interest. It is sufficient to show that $\hat{\vtheta}^{LAP}_q \to \vtheta_q$.  If $\vphi$ is the true parameter of the marginal over $\vx_{A_q}$ in normalized form, \textit{i.e.}\
\begin{align*}
p(\vx_{A_q}\,|\, \vphi) = \sum_{S\setminus A_q} p(\vx_S\,|\,\vtheta) \enspace,
\end{align*}
then it is known that
\begin{align*}
\hat{\vphi}^{ML} \to \vphi
\end{align*}
since
maximum likelihood in consistent under smoothness and identifiability assumptions (for example, see \citet{Fienberg:2012}).  From the LAP argument we see that $\vphi_q = \vtheta_q$ so $\hat{\vphi}^{ML}_q \to \vtheta_q$.
\end{proof}

\subsection{Relationship to maximum likelihood}

Here we prove that, under certain (strong) assumptions, LAP is exactly equal to maximum likelihood.  The main result here will be that under the required assumptions estimation by maximum likelihood and marginalization commute.

Suppose we have a discrete MRF on $\vx_S$ which factorizes according to the cliques $\mathcal{C}$, and let $q \in \mathcal{C}$ be a particular clique of interest.


We will make use of the following characterization of maximum likelihood estimates, which is proved in \cite{jordan2002introduction}.
\begin{lemma}
If a distribution $\hat{p}(\vx_S)$ satisfies that for each $c \in \mathcal{C}$
\begin{align*}
  \hat{p}(\vx_c) = \tilde{p}(\vx_c)
\end{align*}
then $\hat{p}(\vx_S)$ is a maximum likelihood estimate for the empirical distribution $\tilde{p}(\vx_S)$.
\label{lemma:jordan}
\end{lemma}

This characterization allows us to derive an explicit expression for a maximum likelihood estimate of $\hat{p}(\vx_S)$.

\begin{proposition}
The distribution
\begin{align*}
\hat{p}(\vx_S) = \frac{\tilde{p}(\vx_{A_q})\tilde{p}(\vx_{S\setminus
      q})}{\tilde{p}(\vx_{A_q\setminus q})}
\end{align*}
is a maximum likelihood estimate for $\tilde{p}(\vx_S)$.
\label{prop:explicit-ml}
\end{proposition}
\begin{proof}
To see this we compute
\begin{align*}
  \sum_{q} \hat{p}(\vx_{S}) &= \sum_{q} \frac{\tilde{p}(\vx_{A_q})
    \tilde{p}(\vx_{S\setminus q})}{\tilde{p}(\vx_{A_q\setminus q})} = \tilde{p}(\vx_{S\setminus q})
\end{align*}
and
\begin{align*}
  \sum_{S\setminus A_q} \hat{p}(\vx_{S}) &= \sum_{S\setminus A_q}
  \frac{\tilde{p}(\vx_{A_q}) \tilde{p}(\vx_{S\setminus q})}{\tilde{p}(\vx_{A_q\setminus q})} = \tilde{p}(\vx_{A_q})
\end{align*}
For an arbitrary clique $c \in \mathcal{C}$, either $c \subset S\setminus q$ or $c \subset A_q$, and we see that $\hat{f}(x_c) = \tilde{f}(x_c)$ by further marginalizing one of the above expressions.  This shows that the expression we have given for $\hat{p}(\vx_S)$ satisfies the criteria of Lemma~\ref{lemma:jordan}, and is therefore a maximum likelihood estimate for $\tilde{p}(\vx_S)$.
\end{proof}

Suppose we have a family of distributions $\mathcal{F}$ on $\vx_S$ which satisfy the Markov properties of the MRF, and suppose that $\hat{p}(\vx_S) \in \mathcal{F}$ where $\hat{p}(\vx_S)$ is defined as in Proposition~\ref{prop:explicit-ml}.

Define the auxiliary family $\mathcal{F}_q$ associated with the clique $q$ as follows.
\begin{align*}
\mathcal{F}_q = \{ \sum_{S\setminus A_q} p(\vx_S) \,|\, p(\vx_S) \in \mathcal{F}\}
\end{align*}
That is, $\mathcal{F}_q$ is the family of distributions obtained by marginalizing the family $\mathcal{F}$ over $S\setminus A_q$.

\begin{proposition}
The auxiliary family $\mathcal{F}_q$ contains the marginal empirical distribution $\tilde{p}(\vx_{A_q})$.  Moreover $\hat{p}(\vx_{A_q}) = \tilde{p}(\vx_{A_q})$ is a maximum likelihood estimate for $\tilde{p}(\vx_{A_q})$ in $\mathcal{F}_q$.
\label{prop:aux-expressiveness}
\end{proposition}
\begin{proof}
Recall that $\hat{p}(\vx_S)$ from Proposition~\ref{prop:explicit-ml} is in $\mathcal{F}$ by assumption.  Thus,
\begin{align*}
\sum_{S\setminus A_q}\hat{p}(\vx_S) = \tilde{p}(\vx_{A_q})
\end{align*}
is in $\mathcal{F}_q$ by definition.  That $\hat{p}(\vx_{A_q}) \in \mathcal{F}_q$ is a maximum likelihood estimate follows since the log likelihood gradient in Equation~\ref{eq:gradml} is zero when the model and empirical distributions are equal.
\end{proof}

Suppose we can represent the family $\mathcal{F}$ as a Gibbs family, i.e.\
\begin{align*}
\mathcal{F} = \mathcal{F}(\vTheta) = \{ p(\vx_S\,|\,\vtheta) \,|\, \vtheta \in \vTheta \}
\end{align*}
for some domain of parameters $\vTheta$, where
\begin{align*}
p(\vx_S | \vtheta) = \frac{1}{Z(\vtheta)}\exp ( -\sum_{c\in\mathcal{C}} E(\vx_c \,|\, \vtheta_c) ) \enspace.
\end{align*}
Moreover, suppose we have chosen this parameterization so that the potential functions are normalized with respect to zero.

Since $\mathcal{F}$ is representable as a Gibbs family then the auxiliary family $\mathcal{F}_q$ is also representable as a Gibbs family with
\begin{align*}
\mathcal{F}_q = \mathcal{F}_q(\vPhi) = \{ p(\vx_{A_q}\,|\,\vphi) \,|\, \vphi \in \vPhi \}
\end{align*}
for some domain of parameters $\vPhi$.  We will again suppose that this parameterization is chosen so that the potential functions are normalized with respect to zero.

We have already shown that maximum likelihood estimates for $\tilde{p}(\vx_S)$ and $\tilde{p}(\vx_{A_q})$ exist in the families $\mathcal{F}$ and $\mathcal{F}_q$, respectively.  Since we have chosen the parameterizations of these families to be normalized we also have unique maximum likelihood parameters $\hat{\vtheta} \in \vTheta$ and $\hat{\vphi} \in \vPhi$ such that $p(\vx_S\,|\,\hat{\vtheta}) \in \mathcal{F}(\vTheta)$ is a maximum likelihood estimate for $\tilde{p}(\vx_S)$ and $p(\vx_{A_q}\,|\,\hat{\vphi}) \in \mathcal{F}(\vPhi)$ is a maximum likelihood estimate for $\tilde{p}(\vx_{A_q})$.

We can now prove the main result of this section.

\begin{theorem}
Under the assumptions used in this section, estimating the joint parameters by maximum likelihood and integrating the resulting maximum likelihood distribution gives the same result as integrating the joint family of distributions and performing maximum likelihood estimation in the marginal family.  Concisely,
\begin{align*}
\sum_{S\setminus A_q} p(\vx_S \,|\, \hat{\vtheta}) = p(\vx_{A_q} \,|\, \hat{\vphi})
\end{align*}
\label{thm:integrate-estimate}
\end{theorem}
\begin{proof}
We have the following sequence of equalities:
\begin{align*}
p(\vx_S \,|\, \hat{\vtheta})
&\stackrel{(1)}{=} \hat{p}(\vx_S)
\stackrel{(2)}{=} \frac{\tilde{p}(\vx_{A_q})\tilde{p}(\vx_{S\setminus q})}{\tilde{p}(\vx_{A_q\setminus q})}
\\
&\stackrel{(3)}{=} \frac{\hat{p}(\vx_{A_q})\tilde{p}(\vx_{S\setminus q})}{\tilde{p}(\vx_{A_q\setminus q})}
\stackrel{(4)}{=} \frac{p(\vx_{A_q} \,|\, \hat{\vphi})\tilde{p}(\vx_{S\setminus q})}{\tilde{p}(\vx_{A_q\setminus q})}
\end{align*}
The first equality follows from the parameterization of $\mathcal{F}$, the second follows from Proposition~\ref{prop:explicit-ml}, the third from Proposition~\ref{prop:aux-expressiveness} and the fourth follows from the parameterization of $\mathcal{F}_q$.
The theorem is proved by summing both sides of the equality over $S\setminus A_q$.
\end{proof}

Applying the LAP argument to Theorem~\ref{thm:integrate-estimate} we see that $\hat{\vtheta}_q = \hat{\vphi}_q$.

\paragraph{Remark} The assumption that $\hat{p}(\vx_S) \in \mathcal{F}$ amounts to assuming that the \emph{empirical} distribution of the data factors according to the MRF.  This is very unlikely to hold in practice for finite data.  However, if the true model structure is known then this property does hold in the limit of infinite data.

\section{Conclusion}

We have presented a distributed learning algorithm for practical MRFs, where the parameters of each clique can be estimated in different machines. The algorithm is also data efficient in log-linear models, since the estimation of each clique parameter only requires access to local sufficient statistics of the data. Not only are the statistics local to the 1-neighborhoods of each clique, but they can also be precomputed.

Our experiments indicate that the LAP estimators behave similarly to pseudo-likelihood and maximum likelihood for large sample sizes. However, these alternative estimators do not enjoy the same data and model efficiencies as LAP. Finally, we proved that the proposed estimator is consistent.   

The work of Meng and colleagues only considered Gaussian graphical models probably as a result of their linear algebra relaxation techniques. However, it seems feasible to apply their algorithm to the discrete case. A comparison of both techniques is an immediate direction for future work. Combining the different proof techniques of Meng and ours is also of interest. A further addition to the theory would be the derivation of PAC bounds to improve our understanding of the sampling complexity of these estimators. 

This works opens up many directions for future work, including model selection, latent variables, and tied parameters. A distributed implementation on Apache Spark/Hadoop is one of our near goals.

\clearpage
\small{
\bibliography{mrf,mybib}
\bibliographystyle{icml2014}
}

\appendix
\clearpage


\end{document}